\documentclass[12pt]{article}

\usepackage{etex}
\usepackage{enumitem}
\setlist{parsep = -0em, itemsep = 0.25em}
\usepackage[utf8]{inputenc}
\usepackage[dvips,letterpaper,margin=1in]{geometry}
\usepackage{amsmath}
\usepackage{amssymb}
\usepackage{amsthm}
\usepackage{bm}
\usepackage{colonequals}
\usepackage{comment}
\usepackage{graphicx}
\usepackage{subcaption}
\usepackage{xcolor}
\usepackage{tikz} 
\usetikzlibrary{shadows} 
\usepackage{authblk}
\usepackage[utf8]{inputenc} 
\usepackage[T1]{fontenc}    
\usepackage{url}            
\usepackage{booktabs}       
\usepackage{amsfonts}       
\usepackage{nicefrac}       
\usepackage{microtype}      

\usepackage{amsmath,amsthm,amssymb}
\usepackage{algorithm,algpseudocode}
\usepackage[disable]{todonotes}

\usepackage[colorlinks = true,
            linkcolor = teal,
            urlcolor  = blue,
            citecolor = magenta,
            anchorcolor = blue]{hyperref}

\newtheorem{thm}{Theorem}[section]
\newtheorem{defn}[thm]{Definition}
\newtheorem{cor}[thm]{Corollary}
\newtheorem{lemma}[thm]{Lemma}
\newtheorem{prop}[thm]{Proposition}

\def\R{\mathbb{R}}
\def\S{\mathbb{S}}

\def\thetab{\boldsymbol{\theta}}

\def\xb{\boldsymbol{x}}

\newcommand\restr[2]{{
  \left.\kern-\nulldelimiterspace 
  #1 
  \vphantom{\big|} 
  \right|_{#2} 
  }}

\title{On Sparsity in Overparametrised Shallow ReLU Networks}

\author[a]{Jaume de Dios}
\author[b,c,d]{Joan Bruna\thanks{This work is partially supported by the Alfred P. Sloan Foundation, NSF RI-1816753, NSF CAREER CIF 1845360, and the Institute for Advanced Study.}}

\affil[a]{UCLA, Los Angeles, California}
\affil[b]{Courant Institute of Mathematical Sciences, New York
  University, New York}
\affil[c]{Center for Data Science, New York University}
\affil[d]{Institute for Advanced Study, Princeton}

\date{\today}

\begin{document}

\maketitle

\begin{abstract}
The analysis of neural network training beyond their linearization regime remains an outstanding open question, even in the simplest setup of a single hidden-layer. The limit of infinitely wide networks provides an appealing route forward through the mean-field perspective, but a key challenge is to bring learning guarantees back to the finite-neuron setting, where practical algorithms operate. 

Towards closing this gap, and focusing on shallow neural networks, in this work we study the ability of different regularisation 
strategies to capture solutions requiring only a finite amount of neurons, even on the infinitely wide regime. 
Specifically, we consider (i) a form of implicit regularisation obtained by injecting noise into training targets [Blanc et al.~19], and (ii) the variation-norm regularisation [Bach~17], compatible with the mean-field scaling. Under mild assumptions on the activation function (satisfied for instance with ReLUs), we establish that both schemes are minimised by functions having only a finite number of neurons, irrespective of the amount of overparametrisation. We study the consequences of such property and describe the settings where one form of regularisation is favorable over the other.


\end{abstract}
\section{Introduction}


Supervised Learning in high-dimensions remains an active research area, in part due to a confluence of several challenges, including approximation, computational and statistical, where the curse of dimensionality needs to be simultaneously avoided. 
Linear learning based on RKHS theory suffers in the high-dimensional regime due its lack of approximation power, which motivates the study of non-linear learning methods. Amongst them, arguably the simplest is given by single hidden-layer neural networks, trained with some form of gradient descent scheme. Despite its simplicity relative to its deep counterparts, learning with shallow neural networks is still not fully understood from a theoretical standpoint, specifically with regards to giving positive results that explain their advantage over linear learning methods on realistic settings. 

An important theme in recent analysis of neural networks is overparametrisation. By letting the number of neurons grow to infinity, one can obtain optimization, approximation and generalization guarantees beyond the linear learning regime \cite{chizat2020implicit, mei2018mean, rotskoff2018parameters}, at the expense of computational efficiency. In this context, a key goal is to exhibit the right combination of neural architecture and learning algorithm that preserves the guarantees of infinitely-wide neural networks with polynomially-sized networks. 
We note that our study is only tangentially related to the so-called \emph{double-descent} phenomena recently reported in those same models \cite{belkin_reconciling_2019,belkin_two_2019,mei2019generalization}, since we include regularisation from the onset.  
Regularisation can take two flavors: implicit (by injecting noise) or explicit (by adding a term in the empirical loss function). 
They are a key device to give statistical guarantees of generalisation, and in this work we study their potential computational benefit in the overparametrised regime. 

The mean-field regime of overparametrised shallow networks defines solutions in terms of measures over parameters. A natural form of explicit regularisation is then given by the variation-norm \cite{bach2017breaking}, which penalises an continous-equivalent of an $L^1$ norm over weights, and is equivalent to the path-norm \cite{neyshabur_norm-based_2015} and the Barron norm \cite{ma2019barron}, as well as the popular `weight-decay' for homogeneous activations such as the ones we consider here. 
Injecting noise into gradient descent dynamics is another classic form of implicit regularization, under appropriate ergodicity conditions that enable the forgetting of initial conditions. In this work, and inspired by \cite{blanc_implicit_2019}, we focus on a specific form of noise, namely noise added into the training labels. The authors interpreted the noisy gradient dynamics as an approximate Orstein-Uhlenbeck process that `walks' around the manifold of solutions with zero training error (which has positive dimension due to overparametrisation), and `locks' in solutions that minimise a certain implicit regulariser, given by the average squared norm of the network gradients over the data. They leveraged such implicit regularisation to explain the `simple' structure of solutions on specific cases, such as univariate inputs or a single datapoint. 

We encapsulate these two regularisation instances into a general family of regularisers that, combined with mild assumptions on the activation function, produce optimisation problems over the space of probability measures that can only be minimised if the measure has sparse support (Theorem \ref{thm:weight_implies_sparsity}). This gives a novel perspective on the overparametrised regime under these forms of regularisation, that connects it to a finite-dimensional linear program, and whereby the infinitude of parameters required to give optimization guarantees might be relaxed to a finite (albeit possibly exponentially large in dimension) quantity.   






\paragraph{Related Work:}Our work fits into a recent line of work that attempts to uncover the interplay between neural network overparametrisation, optimization and generalisation, dating back at least to \cite{bartlett1997valid}. The implicit bias brought by different flavors of gradient descent in models involving neural networks has since been thoroughly explored \cite{soudry2018implicit,gunasekar2017implicit,savarese2019infinite,neyshabur_norm-based_2015,li2017algorithmic,blanc_implicit_2019,chizat2020implicit,neal2018modern,advani2017high,bos1997dynamics}. Authors have also leveraged the special structure of ReLU activations \cite{maennel_gradient_2018, williams2019gradient}. Closest to our setup are \cite{blanc_implicit_2019,ergen2020convex}, who separately study the two forms of regularisation considered in this work. Our characterisation is consistent with theirs, but extends to a general high-dimensional setting and provides a unified picture.
Finally, recently \cite{bubeck2020network} studied sparse solutions of the TV regularised problem, establishing that in generic data one can find solutions with nearly optimal cost and sparsity; we instead focus on the exact minimisers of the constrained and penalised objectives, but leave the connections for future work.


\section{Preliminaries}







\subsection{Shallow Neural Networks and Variation-Norm Spaces}

We consider a shallow neural network with hidden-layer-width $m$ as the following function:
\begin{equation}
\label{eq:basic}
    f(\xb; \thetab_1\dots \thetab_m) := \frac{1}{m} \sum_{j=1}^m \phi(\thetab_j, \xb),
\end{equation}
where $\xb \in \Omega \subseteq \R^d$ is the feature of the input data, $\phi(\thetab, \xb) = c \sigma(a^\top \xb + b)$ is the \emph{neuron} constructed from an activation function $\sigma$, and $\thetab_j = (c_j, \theta_j) \in \R \times \R^{d+1}:= \Theta$ with $\theta=( a, b )$ are the parameters for the $j$-th neuron. 
As observed by several authors \cite{mei2018mean,chizat2018global, rotskoff2018parameters, sirignano2018dgm}, such a neural network can be equivalently represented in terms of a probability measure over $\Theta$ as $f_{\mu^{(m)}}$, where
\begin{equation}
    f_\mu(\xb) := \int_\Theta \phi(\thetab, \xb) \mu(d\thetab)~,
\end{equation}
and $\mu^{(m)}$ is the empirical measure of the parameters $\{ \thetab_j \}_{j=1}^m$: 
    $\mu^{(m)} = \frac{1}{m} \sum_{j=1}^m \delta_{\thetab_j}~.$ 

Let us denote by $\mathcal{P}(\Theta)$ the space of probability measures over $\Theta$. In the following, we write $\langle \chi, \mu \rangle := \int_\Theta \chi(\thetab) \mu(d\thetab)$ for a Borel-measurable test function $\chi$ \todo{Borel regularity ens val no? Després demanarem convexity pero so far-- de fet, utilitzem aquesta notació at all?} and $\phi_{\xb}:=\phi(\cdot, \xb)$.
Let $V: \Theta \to \R^+$ be a non-negative potential function satisfying $\inf_{\alpha, c} V(\alpha a, \alpha b, c) > 0$. 
 The set of functions representable as (\ref{eq:basic}) defines a metric space $\mathcal{F}_V$ with norm 
\begin{equation}
\label{eq:bli}
    \| f \|_{V} := \inf \left\{ \int_{\Theta}  V(\thetab) \mu(d\thetab);~ f(\xb) = \int_\Theta \phi(\thetab, \xb) \mu(d\thetab)\,. \right\}~.
\end{equation}
When $V(\thetab) = |c| \| \theta \| $, the resulting space is
 the so-called \emph{variation-norm space} $\mathcal{F}_1$  \cite{bach2017breaking}.
 One can verify \cite{ma2019barron} that  $V(\thetab) = |c|^q (\|\theta \| )^q$  gives the same functional characterisation for any $q \geq 1$. Moreover, if $\phi$ is 2-homogeneous w.r.t. $\thetab$ (such as the ReLU $\sigma(t)=\max(0,t)$) \footnote{$\phi(t\thetab, \xb) = t^2 \phi(\thetab, \xb)$}, then one can also verify that $V(\thetab) = \| \thetab\|^2$ also defines the same space \footnote{By first projecting the measure to the unit sphere $\| \theta\|=1$ and then lifting again as a point mass in the $c$-direction, and using the fact that $\mu$ has mass $1$; see \cite{ma2019barron}, Theorem 1.}. Such apparent flexibility is a consequence of the fact that $\mu$, the underlying object parametrising the integral representation (\ref{eq:bli}), is in fact a \emph{lifted} version of a more `fundamental' object $\nu=\mathrm{P}(\mu) \in \mathcal{M}(\mathbb{S}^d)$, the space of signed Radon measures over the unit $(d+1)$-dimensional sphere
 \todo{replace unit sphere by the projective space? the fundamental object here are the space of $d$-dimensional hyperplanes}\todo{Jaume: No: el signe importa per a les ReLu, a diferència de al projecitu.}. The measures $\mu$ and $\nu$ are related via the projection 
 \begin{equation}
     \int_{\mathbb{S}^d} \chi(\tilde{\theta}) d\nu(\tilde{\theta}) = \int_{\Theta} c \| \theta \| \chi\left(\frac{\theta}{\|\theta\|}\right) d\mu(\thetab)
 \end{equation}
 for all continuous test functions $\chi: \S^d \to \R$. One can verify  \cite{chizat2019sparse} that for the previous choice of potential $V$, one has
$\|f \|_V = \inf \{ \| \nu\|_{\mathrm{TV}};\, f(\xb) = \int_{\S^d} \sigma(\langle \tilde{\theta}, \tilde{\xb} \rangle) d\nu(\tilde{\theta}) \}$, where $\tilde{\xb}=(\xb,1)$ and $\| \nu \|_{\mathrm{TV}}$ is the \emph{Total-Variation} norm of $\nu$ \cite{bach2017breaking}.  

This variation-norm space contains any RKHS whose kernel is generated as an expectation over features  $k(\xb, \xb') = \int_\Theta \phi(\thetab,\xb) \phi(\thetab,\xb') \mu_0(d\thetab)$ with a base measure $\mu_0$, but it provides crucial advantages over these RKHS at approximating certain non-smooth, high-dimensional functions having some `hidden' low-dimensional structure \cite{bach2017breaking}. This motivates the study of overparametrised shallow networks with the scaling as in (\ref{eq:basic}), as opposed to the NTK scaling \cite{jacot2018neural} in $1/\sqrt{m}$. 

As we will see in Section \ref{sec:noise}, other $\mathcal{F}_V$ spaces arise naturally when training overparametrised neural networks with noisy dynamics, corresponding to different choice of potential function. In the following, we are going to focus on the ReLU case where $\sigma(t) = (t)_+ = \max(0,t)$. 

\subsection{Empirical Risk Minimization and Representer Theorem}

Given samples $\{ (\xb_i, y_i) \in \Omega \times \R \}_{i=1\dots n}$ for a regression task, learning in $\mathcal{F}_1$ in the \emph{interpolant} regime (corresponding to overparametrised models) amounts to the following problem
\begin{equation}
\label{eq:erm}
    \min_{ f \in \mathcal{F}_V} \| f\|_V ~~s.t. ~~f(\xb_i) = y_i ~\text{ for } i=1\dots n~.
\end{equation}
This problem can be equivalently expressed as an optimization problem in $\mathcal{P}(\Theta)$:
\begin{equation}
\label{eq:ermmeasure}
\min_{ \mu \in \mathcal{P}(\Theta)} \langle V, \mu \rangle ~s.t. ~ \langle \phi_{\xb_i}, \mu \rangle = y_i\,,i=1\dots n~.
\end{equation}
Even though $\mathcal{F}_1$ is not a RKHS, its convex structure also provides a Representer Theorem \cite{zuho1948, fisher1975spline, bach2017breaking, boyer2019representer}: the extreme points of the solution set of (\ref{eq:ermmeasure}) consist of point masses $\sum_{j=1}^r c_j \delta_{\thetab_j}$ of support $r$ at most $n$. In other words, there exist minimisers of the variation norm that require at most $n$ neurons to interpolate the data. However, in contrast with the RKHS representer theorem, these neurons are not explicit in terms of the datapoints, and cannot be directly characterised from (\ref{eq:ermmeasure}) since this problem generally lacks unicity.  

In practice, (\ref{eq:ermmeasure}) is typically solved in the penalised form, by introducing the explicit regularised empirical loss
\begin{equation}
\label{eq:penform}
    \mathcal{L}_\lambda(\mu) = \frac{1}{n} \sum_{i=1}^n \ell( \langle \phi_{\xb_i},\mu \rangle, y_i) + \lambda \langle V, \mu \rangle~,
\end{equation}
where $\ell$ is convex w.r.t. its first argument, e.g. $\ell(x,y) = \frac{1}{2}|x-y|^2$ and $\lambda$ controls the regularisation strength, so (\ref{eq:ermmeasure}) can be obtained as the limit $\lambda \to 0^+$ in (\ref{eq:penform}). 

\subsection{Training Overparametrised Neural Networks and Wasserstein Gradient Flows}

Notice that for empirical measures $\mu^{(m)}$ corresponding to a $m$-width shallow network, the loss $\mathcal{L}(\mu^{(m)})$ is precisely the loss
\begin{equation}
\label{eq:penform2}
{L}(\thetab_1, \dots, \thetab_m) =\frac{1}{n}\sum_i \ell(f(\xb_i;\thetab_1,\dots, \thetab_m), y_i) + \frac{\lambda}{m}\sum_{j=1}^m V(\thetab_j)~,
\end{equation}
with respect to the parameters $\{ \thetab_j \}_{j \leq m}$.    
In that case, the regularisation term corresponds to \emph{weight decay} for $V(\thetab)=\|\thetab\|^2$, and the so-called \emph{path norm} \cite{neyshabur_norm-based_2015} for $V(\thetab)=|c| \| \theta \|$. As also argued in \cite{neyshabur2014search}, 
these two regularisation terms capture the same implicit bias in the functional space.

Performing gradient descent on $L$ with respect to $\{\thetab_j\}_j$ induces  gradient dynamics on the functional $\mathcal{L}$ over $\mathcal{P}(\Theta)$,  with respect to the Wasserstein metric \cite{mei2018mean,chizat2018global, rotskoff2018parameters, sirignano2018dgm},
thus mapping an Euclidean non-convex problem (\ref{eq:penform2}) to a non-Euclidean convex one (\ref{eq:penform}).
The gradient dynamics of overparametrised neural networks in the infinite-width limit can thus be analysed as the mean-field limit of such Wasserstein Gradient flows, leading to global convergence for problems such as (\ref{eq:penform}) under appropriate homogeneity assumptions for $\phi$ \cite{chizat2018global, chizat2019sparse}. However, such convergence results are qualitative and only hold in the limit of $m \to \infty$ without extra structure in the minimisers of $\mathcal{L}(\mu)$ such as being point masses -- precisely the structure that our main results from Section \ref{sec:tv} provide. 

\subsection{Implicit Regularization}

A popular alternative to the explicit variation-norm regularisation is to rely on the particular choice of optimization scheme to provide an implicit regularisation effect. While gradient descent enjoys powerful implicit regularisation for logistic regression \cite{soudry2018implicit, chizat2020implicit} by favoring max-margin solutions, and for linear least-squares regression by choosing solutions of minimal $\ell_2$ norm, the situation for non-linear least-squares is less clear. 

Indeed, in those cases, the dynamics do not typically forget the initial conditions, making them very sensitive to initialisation and producing substantial qualitative changes, as illustrated by the lazy dynamics \cite{chizat2018note}. A powerful alternative is thus given by algorithms that  inject noise into the dynamics, such as SGD or Langevin Dynamics. 
The latter was studied in the framework of overparametrised shallow networks in \cite{mei2018mean}. 
This dynamics leads to an associated cost of the form:
 \begin{equation}
    \min_\mu \frac{1}{n} \sum_{i=1}^n \ell( \langle \phi_{\xb_i},\mu \rangle, y_i) - \lambda \mathcal{H}(\mu)~,
\end{equation}
where $\mathcal{H}$ is the relative entropy of $\mu$. Minimizers to this functional have a smooth density that is solution to the associated elliptic Euler-Lagrange problem.
They are therefore not sparse due to the entropic regularisation. 

An alternative to Langevin Dynamics and SGD was recently studied in \cite{blanc_implicit_2019}, consisting of adding noise into the targets $y_i$ rather than the parameters. As already hinted in \cite{blanc_implicit_2019}, the resulting dynamics capture an implicit bias corresponding to a $V$-space for a \emph{data-dependent} potential function $V$.


\section{Implicit Regularisation by Label Noise}
\label{sec:noise}

The goal of this section is to establish a link between a type of noise-induced regularisation phenomena established in \cite[Theorem 2]{blanc_implicit_2019} and the variation-norm spaces defined in section 3.1.

\subsection{Results in Blanc. et. al.}

Blanc et al. \cite{blanc_implicit_2019} propose Algorithm \ref{alg:noisy_sgd} for training a neural network on a data-set $(\xb_i, y_i)_{i=1}^{n}$. 
The algorithm is an SGD-type algorithm with a small, independent perturbation on each label $y_i$ sampled independently on each gradient step.

Denote by ${\mathcal{D}}=(\xb_1, \dots \xb_n)$ the input data and let ${\mathbb E}_{{\mathcal{D}}}$ denote the empirical average over the observed $n$ datapoints.  
Let $\vec \thetab=(\thetab_1, \dots, \thetab_m)$. 
The main result therein \cite[Theorem 2]{blanc_implicit_2019} states that, for large times (depending, amongst others, on the noise size and the number of parameters), the solution concentrates at local minimizers (amongst the zero-error solutions) of the \emph{effective regularization loss}:
\begin{equation}
\label{eq:blancminimizer}
	\Lambda(\vec \thetab, \mathcal{D}):= {\mathbb E}_{{\mathcal{D}}} \left [\sum_{j} |\nabla_{\thetab_j}f(\xb,\vec\thetab)|^2\right] \propto \frac 1 m \sum_{j} {\mathbb E}_{{\mathcal{D}}} \left [|\nabla_{\thetab_j}\phi(\thetab_j,\xb)|^2\right]~.
\end{equation}

The RHS of \eqref{eq:blancminimizer} is now an integral over the counting measure $\mu = \frac 1 m \sum_{j=1}^m \delta_{\thetab_j}$, and therefore, 
\begin{equation}
\label{eq:loss_4.1}
	\Lambda(\vec \thetab, \mathcal{D}) \propto \int {\mathbb E}_{{\mathcal{D}}} \left [|\nabla_{\thetab_i}\phi(\thetab_i,\xb)|^2\right] d\mu(\thetab) := \int \widetilde{V}_b(\thetab) d\mu(\thetab)~.
\end{equation}

	\begin{algorithm}[tb]
        \caption{SGD with noise on the labels}
       \label{alg:noisy_sgd}
    \begin{algorithmic}
       \State {\bfseries Parameters:} Noise size $\eta$,  step size $\epsilon$, noise distribution $\rho$
       \State {\bfseries Input:} Data $(\xb_i,y_i)_{i=1}^n$, Initial parameters $\theta_0$ 
       \Repeat
       \State {Choose a random data point $(\xb_{i_t}, y_{i_t})$}
       \State {Generate noise $r \sim \rho$}
       \State $\tilde y_{i_t} := y_{i_t} + \eta r$
       \State $	\thetab_{t+1} = \thetab_t - \epsilon \nabla_{\thetab} [(f(\xb_{i_t};\thetab)-\tilde y_{i_t})^2] $ 
       \Until{Convergence}
    \end{algorithmic}
    \end{algorithm}

\subsection{An equivalent formulation in the ReLU case}
\label{sec:equivalent_relu}
As we have seen in the previous section, the training algorithm with label noise minimizes $\langle \tilde{V}_b, \mu \rangle$. 
In the ReLU case we can explicitly write:
\begin{equation}
    \widetilde{V}_b(a,b,c) = {\mathbb E}_{{\mathcal{D}}} \left[\frac{|c^2|(a^\top \xb+b)^2_+}{c^2}+(1+|\xb|^2) c^2 \chi_{(a^\top \xb+b)\ge 0}\right ]~.
\end{equation}


The key realization is that in order to minimize the value of $\ref{eq:loss_4.1}$, the algorithm will \textit{exploit} the degeneracy $(a,b,c) \sim (\lambda a,\lambda b,\lambda^{-1} c)$ for $\lambda>0$. In other words, we may substitute $\widetilde{V}_b$ for 
$V_b(a,b,c) = \min_{\lambda > 0} \widetilde{V}_b  (\lambda a,\lambda b,\lambda^{-1} c)$ 
without effectively changing the problem. A straightforward computation shows that $V_b(a,b,c)$ is equal to:
\begin{equation}
	V_b(a,b,c) = |c|\sqrt{  {\mathbb E}_{\mathcal D} [ (a^\top \xb+b)_+^2] } := |c|w_{b}(a,b)~.
\end{equation}
There are two crucial properties of the potential function $V_b$. 
First, that it is invariant under the $\lambda$ degeneracy described above: If we have a measure $\mu$ and modify it with the method above to obtain $\tilde \mu$, both of them will have the same loss. The second is that the function $w$ is convex, and, as long as the neurons can `see' enough of the dataset, essentially convex. To make this precise, we start with two definitions:

\begin{defn}
    A neuron with parameters $(a,b)$ is \emph{at the bulk of the dataset} $\mathcal{D}=(\xb_1, \dots, \xb_n)$ whenever the map $(a,b) \mapsto ((a\cdot \xb_1+b)_+, (a\cdot \xb_2+b)_+, \dots (a\cdot \xb_n+b)_+ )$ is locally injective.
\end{defn}

\begin{defn}
	A weight $w(a,b)$ that is 1-homogeneous in $(a,b)$ is \emph{effectively strictly convex} if whenever
		\begin{equation}
			w(\lambda a+(1-\lambda) a', \lambda b+(1-\lambda) b') = \lambda w(a,b)+(1-\lambda)w(a',b')
		\end{equation}
		holds for some parameters $a,a',b,b'$ and $\lambda\in(0,1)$, the parameters $(a,b)$ and $(a',b')$ are proportional to each other.
\end{defn}

With these definitions in hand, we can state the main result of this section, a convexity result for the regularization loss (see the Appendix for a proof):

\begin{prop}
\label{lem:noise_convexity}
	For any set of variables $\{\xb_i\}_{i=1}^n$ the associated weight 
	$
	    w_b(a,b):= \sqrt{ \mathbb{E}_{\mathcal{D}}[ (a^\top\xb+b)_+^2] }
	$ is convex in $(a,b)$. Moreover, if the neuron with parameters $(a,b)$ is at the bulk of the dataset, the weight $w(a,b)$ is effectively strictly convex in a neighbourhood of $(a,b)$.
\end{prop}

We will show next how such convexity structure in the potential function can be leveraged to characterise minimisers of (\ref{eq:ermmeasure}) and (\ref{eq:penform}).  
\section{Minimizers to the convex-type TV have sparse support}
\label{sec:tv}

	In this section we will first focus on understanding the solutions of the regularised interpolation problem (\ref{eq:ermmeasure}), for potential functions $V$ of the form $V(\thetab) = |c| w(a,b)$ for a given regularization weight $w$. As discussed so far, this problem is motivated by the case $w = \sqrt{\|a\|^2+b^2}$, corresponding to variation-norm learning \cite{bach2017breaking} (and which, as seen in \cite{ongie_function_2019} is also related to a Tikhonov-type regularization). Another example is the implicit regularization scheme discussed in section \ref{sec:noise}. 

	By the rescaling trick (see the discussion in section  \ref{sec:equivalent_relu}) it suffices to consider the case where $w(a,b)$ is homogeneous of degree $1$ in the sense that for $\lambda > 0$ we have $w(\lambda a, \lambda b) = \lambda w (a,b)$. These weights can never be strictly convex (a desirable quantity when looking for uniqueness or sparsity of limits) because homogeneous functions are never strictly convex: $w(2a,2b) = \frac 1 2 (w(a,b)+ w(3a,3b))$. This obstruction, however, appears only from the \emph{parameter} perspective but not from the \emph{neuron} perspective, because proportional parameters also represent proportional neurons. This is a motivation for the \emph{effective convexity} definition in section \ref{sec:equivalent_relu}.

	 Lemma \ref{lem:noise_convexity} in section \ref{sec:equivalent_relu} shows that the implicit regularization term that appeared in \cite{blanc_implicit_2019} is also effectively strictly convex. The fact that the TV norm is effectively strictly convex is a direct consequence of Minkowski's inequality.
	 The definition of \emph{effective} strict convexity motivates an analog definition of sparsity on lifted measures:

	\begin{defn}
	\label{def:effsupp}
		A lifted measure $\mu(a,b,c)$ is effectively supported on $n$ points if the support of $\mu$ is contained on the union of at most $n$ half-planes of the form:
		\begin{equation}
			H_{a,b} = \{(\beta a, \beta b, c), c\in \mathbb R,\beta \in\mathbb R_{\ge 0}\}
		\end{equation}
	\end{defn}

	If a lifted measure is effectively supported on $n$ points, then (by construction) there is a lifted measure supported on exactly $n$ points that represents the same function. With this definition in hand we are ready to state the main result of this section (see Figure 1 for an illustration motivating the theorem):

	\begin{thm}
		\label{thm:weight_implies_sparsity}
		Let $w(a,b)$ be a 1-homogeneous effectively strictly convex regularization weight and $V(\thetab) = |c| w(a,b)$ the associated potential. Then the solutions of the problem (\ref{eq:ermmeasure}) are all effectively supported on a finite number of points. The number of points in the solution is bounded by $O(n)^{O(d)}$. 

	\end{thm}

	The proof of the theorem has two steps. In the first step we will split the parameter space $(a,b)$ into cells where the parameters interact in a particularly nice way. This cell decomposition is reminiscent of the ``open sectors" described in \cite{maennel_gradient_2018}\footnote{The cells defined in this work are the closure of the open sectors defined by Maennel et. al.}. The second step is to use the effective strict convexity result to show that, on each of those cells, optimizing measures are supported on at most one point.

	\paragraph{Cell decomposition of the parameter space}

	Fix a dataset $\mathcal{D}=\{\xb_i\}_{i=1}^n$, with  $\xb_i \in \mathbb R^{d}$. This will split the parameter space of $\theta$ into a finite amount of convex cones $C_1, \dots C_S$ that we will call cells, such that two neurons belong to the same cell iff they are active on the same datapoints. We will consider closed cells, and therefore there will be overlap at the boundary of the cells. By interpreting each datapoint with is dual hyperplane, there is a one-to-one correspondence between these cells and the resulting hyperplane arrangement, thus $S \simeq n^d$ in generic datasets. 

	A key property of the cells is that the map $\theta=(a,b)\mapsto \sigma(\theta, \xb)=(a^\top  \xb +b)_+$ is linear when restricted to a single cell. From the convexity, the cone property, and the restricted linearity property, one can extract the following result:

	\begin{lemma}[Discrete version]
	\label{lem:didi}
	If two parameters $ (a,b),(a',b')$ belong to the same cell, then for all the datapoints $\xb_i$ it holds that
	\begin{equation}
	 \label{eq:effective_linearity}
	 ( a^\top \xb_i + b)_++((a')^\top \xb_i + b')_+ = ((a+a')^\top \xb_i+(b+b'))_+~.
	\end{equation}
	\end{lemma}
	Using induction, plus standard limiting arguments this lemma can be upgraded to the measure-valued case:
	\addtocounter{thm}{-1}
	\begin{lemma}[Continuous version]
	If two measures $\mu(a,b,c),\mu'(a,b,c)$ are supported on the same cell $C$, then
	\begin{equation}
	    \langle \phi_{\xb_i}, \mu \rangle = \langle \phi_{\xb_i}, \mu' \rangle
	\end{equation}
	for all data points $\xb_i$ whenever the moment estimates hold:
	\begin{equation}
	\label{eq:moment_estimates}
		\begin{cases}
			\int_{C} c a\,d\mu(a,b,c) = \int_{C} c a\,d\mu'(a,b,c) & (\text{moment estimate for }a)\\
			\int_{C} cb\,d\mu(a,b,c) = \int_{C} c b\,d\mu'(a,b,c) & (\text{moment estimate for }b)\\
		\end{cases}
	\end{equation}
	\end{lemma}

	\paragraph{Effective convexity - measure case.}  \todo{He canviat aquest a paragraf x consistencia OK}
	Now we want to show that the effective convexity concentrates the mass on each cell. We will first study measures that are already concentrated on a single cell, and then patch the result.
	A first step is that effective convexity for $w(a,b)$ induces a similar notion effective convexity on $|c|w(a,b)$:

	\begin{lemma}
		[Lifting of the effective convexity]
		Let $w(a,b)$ be a 1-homogeneous effectively strictly convex regularization weight. Let $(a,b,c), (a',b',c')$, and $\lambda \in (0,1)$. Then
		\begin{equation}
			|\lambda c + (1-\lambda) c'|w(\lambda a+(1-\lambda) a', \lambda b+(1-\lambda) b') = \lambda |c|w(a,b)+(1-\lambda)|c'|w(a',b')
		\end{equation}
		if and only if there exist constants $\alpha,\beta \in \mathbb R^+$ such that $(a,b,c) = (\alpha a,\alpha b,\beta c)$. 
	\end{lemma}
		In other words, the only scenario where there is no strict convexity is when testing on parameters that represent neurons proportional to each other, since the condition $(a,b,c) = (\alpha a,\alpha b,\beta c)$ is equivalent to saying that the two ReLU neurons share the same hyperplane. 
		
	Convex functions have a unique minimizer, and, in the same way we bootsrapped the discrete, single-variable information into its continuous counterpart in Lemma \ref{lem:didi}, we can see that:

	\begin{lemma}
	\label{lem:one_point_per_cell}
		Let $w(a,b)$ be a 1-homogeneous effectively strictly convex regularization weight, and fix a data cell.
		Amongst all the measures $\nu(a,b)$ which are supported on the fixed cell and have the same moment estimates (eq \eqref{eq:moment_estimates}), the minimizers of $L_{w}(\nu)$ are effectively supported on one point.
	\end{lemma}

	\todo{He descommmentat aquesta linea, crec que si no queda molt sec, no?}Lemma \ref{lem:one_point_per_cell} is the key ingredient for the proof of Theorem \ref{thm:weight_implies_sparsity}:\begin{proof}[Proof of Theorem \ref{thm:weight_implies_sparsity}]
		We can write our measure $\mu = \sum_{i \in \text{Cells}} \mu_i$ where each $\mu_i$ is supported on the closure of a cell. By the previous results, on each cell the associated measure $\mu_i$ has to be supported at at most one point. 
	\end{proof}

\begin{figure}
    \centering
    \begin{minipage}[t]{.3\textwidth}
        \includegraphics[width=\textwidth]{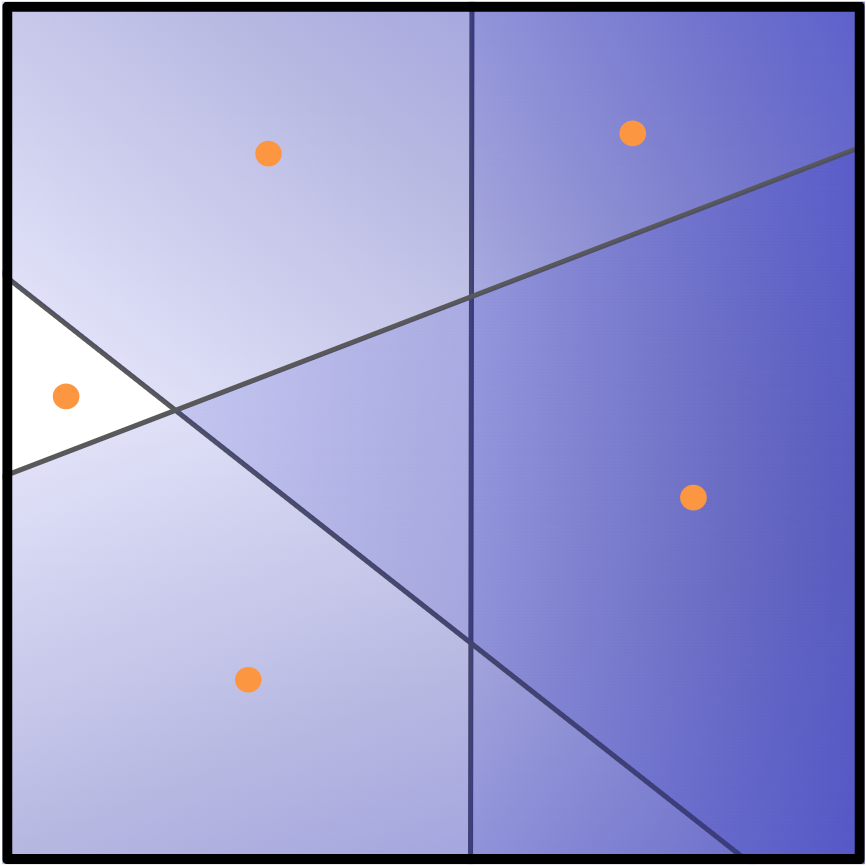}
        \footnotesize
        Figure 1.a: \textbf{Data-space perspective:} Each ReLu separates the data space into two half-spaces, the active half-space and the non-active. Given a subset $\mathcal{D'} \subseteq \mathcal D$ of datapoints (in orange)  there is at most one ReLu active at all points of $\mathcal D'$ but non-active in all points in $\mathcal D\setminus\mathcal D'$. 
    \end{minipage}
    \hspace{.02\textwidth}
    \begin{minipage}[t]{.3\textwidth}
        \includegraphics[width=\textwidth]{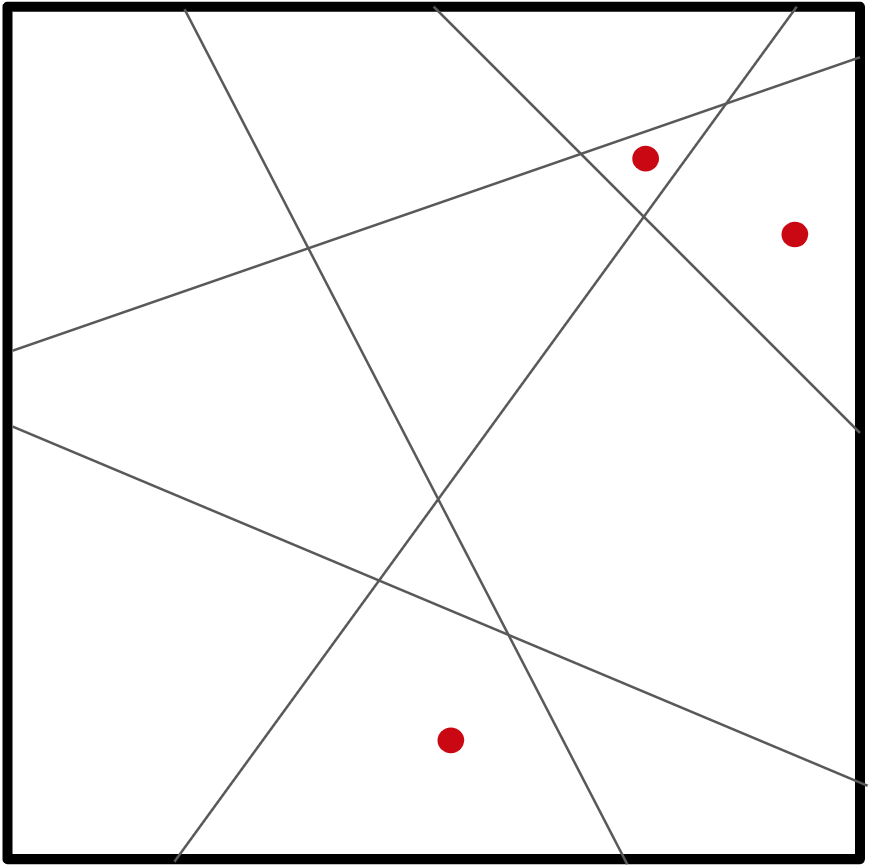}
        \footnotesize
        Figure 1.b: \textbf{Parameter-space perspective:} Each data-point induces a hyperplane on the parameter space: the set of ReLu parameters that have the datapoint at the hinge. There is at most one neuron (in red)  on each cell generated by these hyperplanes (See 1.c).
    \end{minipage}
    \hspace{.02\textwidth}
    \begin{minipage}[t]{.3\textwidth}
        \includegraphics[width=\textwidth]{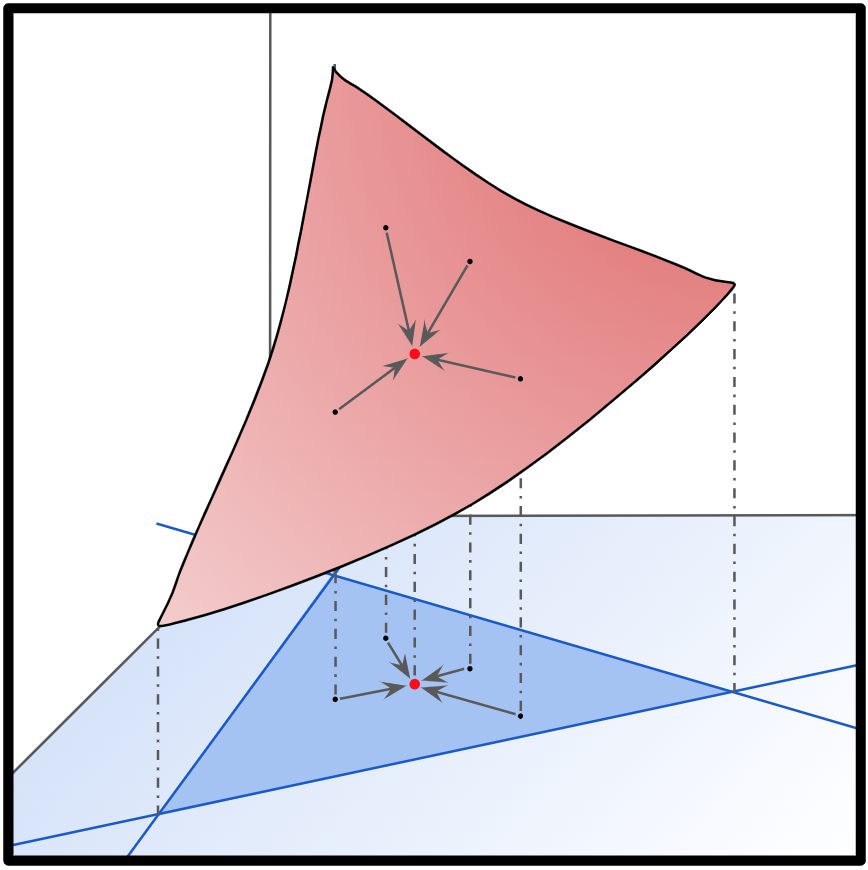}
        \footnotesize
        Figure 1.c: \textbf{Single-cell optimizers:} If a cell has more than one neuron on a cell, substituting all the neurons on the cell by a neuron in the center of mass will not affect the predictions on the data, but will give a gain on the regularizer. 
    \end{minipage}
    \vspace{-.1in}
\end{figure}

\paragraph{Recovering Discrete Minimisers by Projecting into the Sphere: }
Since the constrained optimization problem (\ref{eq:ermmeasure}) is defined over probability measures in the lifted space $\Theta$, there is a price to pay to recover discrete minimizers, what we called `effective' support in definition \ref{def:effsupp}. In the TV case, this aspect disappears as one looks at the equivalent problem in the space $\mathcal{M}(\S^d)$, as shown in the following corollary:
\begin{cor}
\label{cor:radonversion}
Let 
\begin{equation}
\label{eq:ermradon}
    \min_{\nu \in \mathcal{M}(\S^d)} \| \nu \|_{\mathrm{TV}} ~s.t.~\int_{\S^d} \sigma(\langle \tilde{\xb}_i, {\theta} \rangle) d\nu(\theta) = y_i ~\text{ for } i=1\dots n~.
\end{equation}
Then the solutions of (\ref{eq:ermradon}) are all supported on a finite number of points, of the same cardinality as in Theorem \ref{thm:support_is_unique}. 
\end{cor}

\paragraph{Unicity of Supports Across all minimisers:}
We have so far proven that solutions to the constrained minimization have a single mass point at each cell. However, a strictly stronger result holds:
\begin{thm}
\label{thm:support_is_unique}
Let $C$ be a cell. Let $\mu, \mu'$ be zero-errror minimizers of $L_{\text{reg}}$ with non-zero support on $C$. Then both $\mu$ and $\mu'$ are effectively supported in the same single point in $C$.
\end{thm}
\begin{proof}[Proof sketch]
By contradiction, let $\mu, \mu'$ have different (effective) support on the same cell. By strict convexity the average of $\mu$ and $\mu'$ would also be a minimizer supported in two points in $C$, but that would contradict Lemma \ref{lem:one_point_per_cell}.
\end{proof}

\paragraph{Penalized minimization problem:}

In a parallel fashion to the problem presented in this section, one can study the \emph{penalized} version (\ref{eq:penform}).  All the sparsity results in this section (Theorems \ref{thm:weight_implies_sparsity}, \ref{thm:support_is_unique} and Corollary \ref{cor:radonversion} in particular, but also suitable variations of the intermediate lemmas) transfer to the convex case for arbitrary $\lambda>0$, as explained in the appendix:
\begin{cor}
\label{cor:penresult}
Under the same assumptions of Theorem \ref{thm:support_is_unique} and for any $\lambda>0$, all minimisers of $\mathcal{L}_\lambda(\mu)$ are effectively supported on a finite number of points. For the corresponding problem expressed in the projected $\mathrm{P}(\mu) \in \mathcal{M}(\S^d)$, all minimisers are supported on a finite number of points. 
\end{cor}

\paragraph{Implicit versus TV regularisation: } We can make the following observations on the relationship between the two forms of regularisation:
\begin{itemize}
\item The implicit regularization is translation invariant while TV regularization is not. One could remove the bias ($b$) dependence on TV (and consider $V(a,b,c) = |c|\|a\|$, but in this case theorems \ref{thm:weight_implies_sparsity}, \ref{thm:support_is_unique} become false). This is not a draw-back for all applications (for example in image processing, since translation corresponds to a color change).
\item Implicit regularization is (in principle) more costly to compute, especially whenever the number datapoints is large. 
\end{itemize}

\paragraph{From discrete solutions to sparse solutions:}

Corollary \ref{cor:radonversion} and Theorem \ref{thm:support_is_unique} allow us to turn the original infinite-dimensional problem into a $S$-dimensional linear problem with linear constraints, where the only free parameters are the the amount of (signed) \emph{mass} that is given to each cell. Recall that $S$ corresponds to the number of cells of the hyperplane arrangement determined by the datapoints in the projective space. 

Focusing on the TV case, let 
$\nu = \sum_{s=1}^S z_s \delta_{\theta_s} \in \mathcal{M}(\S^d)$
be a minimiser of (\ref{eq:ermradon}). Observe that as a consequence 
of Theorem \ref{thm:support_is_unique}, the locations $\theta_1 \dots \theta_S$ are shared across every minimiser (but unknown a priori).
Its coefficients $z=(z_1\dots z_S)$ are thus solutions of the linear program
\begin{equation}
\label{eq:finitedim}
    \min \| z \|_1 ~s.t.~\mathcal{A}z = y~,
\end{equation}
where $\mathcal{A} \in \R^{n \times S}$ is the matrix $\mathcal{A}_{i,s} = (\langle \tilde{\xb}_i, \theta_s \rangle)_+$. 
From the Representer theorem, we know that there exists a solution $z^*$ of (\ref{eq:finitedim}) of support $\| z^* \|_0 \leq n \ll S$, and in the generic case this solution is the sparsest one, ie $z^*$ solves the $\ell_0$ minimisation counterpart of (\ref{eq:finitedim}). 

A natural question is thus to understand under what conditions on the data the $\ell_1$ and $\ell_0$ problems share the same minimisers. This is an instance of Compressed Sensing, which could be analysed by studying properties of the sensing matrix $\mathcal{A}$ such as the RIP (leveraging random datasets) or the Nullspace property. This question has been also recently studied in \cite{ergen2020convex}. This is out of the scope of the present work and is left for future research.  




\paragraph{Consequences for Gradient-Descent Optimization:}

With abuse of notation, let
\begin{equation}
    \mathcal{L}(\nu) = \frac{1}{n} \sum_{i=1}^n \left| \int_{\S^d} \sigma(\langle \tilde{\xb}_i, \theta\rangle ) d\nu(\theta) - y_i \right|^2 + \lambda \| \nu \|_{\mathrm{TV}}
\end{equation}
be the penalized form of the regularised problem, expressed in terms of the Radon measure in $\S^d$. The ability of gradient-descent to minimise such problems was recently studied in \cite{chizat2019sparse}, leading to a positive result under cetain assumptions, notably that minimisers are discrete and `feel'  a positive curvature around their support. 
The local convexity analysis that resulted in the discrete characterisation of solutions for this penalised problem (Lemma \ref{lem:one_point_per_cell}) brings precisely this geometric property around minimisers, suggesting 
a favorable case for gradient-descent on sufficiently overparametrised (but finite) shallow ReLU networks. However, we note that (i) our ReLU setup does not have appropriate smoothness to directly apply the results from \cite{chizat2019sparse}, and (ii) the finite-particle guarantees would still be at best exponential in the dimension. We leave these considerations for future work. 




\section{Conclusions}
In this work, we have studied the structure of the minimisers in overparametrised shallow RelU networks under a broad umbrella of implicit and explicit regularisation strategies, including the popular weight decay and noisy gradient dynamics through label noise. Our main result leverages the underlying convexity generated by the regularisation to show that the resulting learning objectives only admit sparse minimisers, irrespective of the amount of overparametrisation. 

This fact effectively maps the task of learning in variation-norm ReLU spaces (which are infinite-dimensional) into a finite-dimensional linear program. However, such linear program is determined by the hyperplane arrangement generated by the finite dataset after dualization -- and thus of size exponential in dimension. As such, it does not directly solve the computational barrier present in the mean-field formulations of sparse measure optimization \cite{chizat2019sparse}, which is unavoidable in certain data regimes under the general SQ-model \cite{klivans2020}. 

As future work, it would be interesting to further explore the connections with the recent `memorization' constructions of \cite{bubeck2020network}, and to study favorable conditions on the datapoints that could break the exponential size of the linear program.

\bibliographystyle{alpha}
\bibliography{bibliography}

\newcommand{\etalchar}[1]{$^{#1}$}
\begin{thebibliography}{GWB{\etalchar{+}}17}

\bibitem[AS17]{advani2017high}
Madhu~S Advani and Andrew~M Saxe.
\newblock High-dimensional dynamics of generalization error in neural networks.
\newblock {\em arXiv preprint arXiv:1710.03667}, 2017.

\bibitem[Bac17]{bach2017breaking}
Francis Bach.
\newblock Breaking the curse of dimensionality with convex neural networks.
\newblock {\em The Journal of Machine Learning Research}, 18(1):629--681, 2017.

\bibitem[Bar97]{bartlett1997valid}
Peter~L Bartlett.
\newblock For valid generalization the size of the weights is more important
  than the size of the network.
\newblock In {\em Advances in neural information processing systems}, pages
  134--140, 1997.

\bibitem[BCC{\etalchar{+}}19]{boyer2019representer}
Claire Boyer, Antonin Chambolle, Yohann~De Castro, Vincent Duval,
  Fr{\'e}d{\'e}ric De~Gournay, and Pierre Weiss.
\newblock On representer theorems and convex regularization.
\newblock {\em SIAM Journal on Optimization}, 29(2):1260--1281, 2019.

\bibitem[BELM20]{bubeck2020network}
S{\'e}bastien Bubeck, Ronen Eldan, Yin~Tat Lee, and Dan Mikulincer.
\newblock Network size and weights size for memorization with two-layers neural
  networks.
\newblock {\em arXiv preprint arXiv:2006.02855}, 2020.

\bibitem[BGVV19]{blanc_implicit_2019}
Guy Blanc, Neha Gupta, Gregory Valiant, and Paul Valiant.
\newblock Implicit regularization for deep neural networks driven by an
  {Ornstein}-{Uhlenbeck} like process.
\newblock {\em arXiv:1904.09080 [cs, stat]}, April 2019.
\newblock arXiv: 1904.09080.

\bibitem[BHMM19]{belkin_reconciling_2019}
Mikhail Belkin, Daniel Hsu, Siyuan Ma, and Soumik Mandal.
\newblock Reconciling modern machine learning practice and the bias-variance
  trade-off.
\newblock {\em arXiv:1812.11118 [cs, stat]}, September 2019.
\newblock arXiv: 1812.11118.

\bibitem[BHX19]{belkin_two_2019}
Mikhail Belkin, Daniel Hsu, and Ji~Xu.
\newblock Two models of double descent for weak features.
\newblock {\em arXiv:1903.07571 [cs, stat]}, March 2019.
\newblock arXiv: 1903.07571.

\bibitem[BO97]{bos1997dynamics}
Siegfried B{\"o}s and Manfred Opper.
\newblock Dynamics of training.
\newblock In {\em Advances in Neural Information Processing Systems}, pages
  141--147, 1997.

\bibitem[CB18a]{chizat2018note}
Lenaic Chizat and Francis Bach.
\newblock A note on lazy training in supervised differentiable programming.
\newblock {\em arXiv preprint arXiv:1812.07956}, 2018.

\bibitem[CB18b]{chizat2018global}
Lenaic Chizat and Francis Bach.
\newblock On the global convergence of gradient descent for over-parameterized
  models using optimal transport.
\newblock In {\em Advances in Neural Information Processing Systems}, pages
  3036--3046, 2018.

\bibitem[CB20]{chizat2020implicit}
L{\'e}na{\"\i}c Chizat and Francis Bach.
\newblock Implicit bias of gradient descent for wide two-layer neural networks
  trained with the logistic loss.
\newblock {\em arXiv preprint arXiv:2002.04486}, 2020.

\bibitem[Chi19]{chizat2019sparse}
Lenaic Chizat.
\newblock Sparse optimization on measures with over-parameterized gradient
  descent.
\newblock {\em arXiv preprint arXiv:1907.10300}, 2019.

\bibitem[EP20]{ergen2020convex}
Tolga Ergen and Mert Pilanci.
\newblock Convex geometry and duality of over-parameterized neural networks.
\newblock {\em arXiv preprint arXiv:2002.11219}, 2020.

\bibitem[FJ75]{fisher1975spline}
SD~Fisher and Joseph~W Jerome.
\newblock Spline solutions to l1 extremal problems in one and several
  variables.
\newblock {\em Journal of Approximation Theory}, 13(1):73--83, 1975.

\bibitem[GGJ{\etalchar{+}}20]{klivans2020}
S.~Goel, A.~Gollakota, Z.~Jin, S.~Karmalkar, and A.~Klivans.
\newblock Superpolynomial lower bounds for learning one-layer neural networks
  using gradient descent.
\newblock {\em ICML}, 2020.

\bibitem[GWB{\etalchar{+}}17]{gunasekar2017implicit}
Suriya Gunasekar, Blake~E Woodworth, Srinadh Bhojanapalli, Behnam Neyshabur,
  and Nati Srebro.
\newblock Implicit regularization in matrix factorization.
\newblock In {\em Advances in Neural Information Processing Systems}, pages
  6151--6159, 2017.

\bibitem[JGH18]{jacot2018neural}
Arthur Jacot, Franck Gabriel, and Cl{\'e}ment Hongler.
\newblock Neural tangent kernel: Convergence and generalization in neural
  networks.
\newblock In {\em Advances in neural information processing systems}, pages
  8571--8580, 2018.

\bibitem[LMZ17]{li2017algorithmic}
Yuanzhi Li, Tengyu Ma, and Hongyang Zhang.
\newblock Algorithmic regularization in over-parameterized matrix sensing and
  neural networks with quadratic activations.
\newblock {\em arXiv preprint arXiv:1712.09203}, 2017.

\bibitem[MBG18]{maennel_gradient_2018}
Hartmut Maennel, Olivier Bousquet, and Sylvain Gelly.
\newblock Gradient {Descent} {Quantizes} {ReLU} {Network} {Features}.
\newblock {\em arXiv:1803.08367 [cs, stat]}, March 2018.
\newblock arXiv: 1803.08367.

\bibitem[MM19]{mei2019generalization}
Song Mei and Andrea Montanari.
\newblock The generalization error of random features regression: Precise
  asymptotics and double descent curve.
\newblock {\em arXiv preprint arXiv:1908.05355}, 2019.

\bibitem[MMN18]{mei2018mean}
Song Mei, Andrea Montanari, and Phan-Minh Nguyen.
\newblock A mean field view of the landscape of two-layer neural networks.
\newblock {\em Proceedings of the National Academy of Sciences},
  115(33):E7665--E7671, 2018.

\bibitem[MWE19]{ma2019barron}
Chao Ma, Lei Wu, and Weinan E.
\newblock Barron spaces and the compositional function spaces for neural
  network models.
\newblock {\em arXiv preprint arXiv:1906.08039}, 2019.

\bibitem[NMB{\etalchar{+}}18]{neal2018modern}
Brady Neal, Sarthak Mittal, Aristide Baratin, Vinayak Tantia, Matthew Scicluna,
  Simon Lacoste-Julien, and Ioannis Mitliagkas.
\newblock A modern take on the bias-variance tradeoff in neural networks.
\newblock {\em arXiv preprint arXiv:1810.08591}, 2018.

\bibitem[NTS14]{neyshabur2014search}
Behnam Neyshabur, Ryota Tomioka, and Nathan Srebro.
\newblock In search of the real inductive bias: On the role of implicit
  regularization in deep learning.
\newblock {\em arXiv preprint arXiv:1412.6614}, 2014.

\bibitem[NTS15]{neyshabur_norm-based_2015}
Behnam Neyshabur, Ryota Tomioka, and Nathan Srebro.
\newblock Norm-{Based} {Capacity} {Control} in {Neural} {Networks}.
\newblock {\em arXiv:1503.00036 [cs, stat]}, April 2015.
\newblock arXiv: 1503.00036.

\bibitem[OWSS19]{ongie_function_2019}
Greg Ongie, Rebecca Willett, Daniel Soudry, and Nathan Srebro.
\newblock A {Function} {Space} {View} of {Bounded} {Norm} {Infinite} {Width}
  {ReLU} {Nets}: {The} {Multivariate} {Case}.
\newblock {\em arXiv:1910.01635 [cs, stat]}, October 2019.
\newblock arXiv: 1910.01635.

\bibitem[RVE18]{rotskoff2018parameters}
Grant Rotskoff and Eric Vanden-Eijnden.
\newblock Parameters as interacting particles: long time convergence and
  asymptotic error scaling of neural networks.
\newblock In {\em Advances in Neural Information Processing Systems}, pages
  7146--7155, 2018.

\bibitem[SESS19]{savarese2019infinite}
Pedro Savarese, Itay Evron, Daniel Soudry, and Nathan Srebro.
\newblock How do infinite width bounded norm networks look in function space?
\newblock In {\em Conference on Learning Theory}, pages 2667--2690, 2019.

\bibitem[SHN{\etalchar{+}}18]{soudry2018implicit}
Daniel Soudry, Elad Hoffer, Mor~Shpigel Nacson, Suriya Gunasekar, and Nathan
  Srebro.
\newblock The implicit bias of gradient descent on separable data.
\newblock {\em The Journal of Machine Learning Research}, 19(1):2822--2878,
  2018.

\bibitem[SS18]{sirignano2018dgm}
Justin Sirignano and Konstantinos Spiliopoulos.
\newblock Dgm: A deep learning algorithm for solving partial differential
  equations.
\newblock {\em Journal of Computational Physics}, 375:1339--1364, 2018.

\bibitem[WTP{\etalchar{+}}19]{williams2019gradient}
Francis Williams, Matthew Trager, Daniele Panozzo, Claudio Silva, Denis Zorin,
  and Joan Bruna.
\newblock Gradient dynamics of shallow univariate relu networks.
\newblock In {\em Advances in Neural Information Processing Systems}, pages
  8376--8385, 2019.

\bibitem[Zuh48]{zuho1948}
S.~Zuhovickii.
\newblock Remarks on problems in approximation theory.
\newblock {\em Mat. Zbirnik KDU}, 1948.

\end{thebibliography}

\newpage

\appendix

\section{Additional Proofs}

\paragraph{Proof of Proposition 3.3.}

\begin{proof}
 	Convexity follows from the fact that $l^2$ average of convex functions is convex:

 	 \begin{align}
 	 	\left \|\left(  \left(\frac{a+a'}2\right)^\top \xb_j  +\frac{b+b'}2\right)_+ \right \|_{l^2(j)}
 	 	\le &
 	 	 \left \|\frac 1 2 (a^\top \xb_j+b)_+ + \frac 1 2 ((a')^\top \xb_j+b')_+ \right \|_{l^2(j)}
 	 	 \\\le &
 	 	  \frac 1 2
 	 	  \| (a^\top \xb_j+b)_+   \|_{l^2(j)}
 	 	  +
 	 	  \frac 1 2
 	 	  \| ((a')^\top \xb_j+b')_+ \|_{l^2(j)} \nonumber
 	 \end{align}
 	 Where the first line folows from the convexity of the ReLu. Strict convexity follows similarly: Local injectivity implies that, unless $(a,b)$ is proportional to $(a',b')$ the vector $[(a^\top \xb_j + b)_+]_{j=1}^n$ is not proportional to the vector $[((a')^\top \xb_j + b')_+]_{j=1}^n$. In that case by (sharp) Minkowski inequality, the inequality has to be strict.
 \end{proof}
\paragraph{Proof of Lemma 4.3 (Discrete version)}
 
\begin{proof}
 Let $x_i$ be a data-point. If $(a,b)$ and $(a',b')$ belong to the same cell, there are two options: \\

Both $(ax_i+b)\le 0$ and $(a'x_i+b) \le 0$. In this case, $0 = (ax_i+b)_+ + (a'x_i+b')_+ = ((a+a')x_i+b+b')$.\\
Both $(ax_i+b)\ge 0$ and $(a'x_i+b) \ge 0$. In this case, we have $(ax_i+b)_+ + (a'x_i+b')_+  = (ax_i+b)+(ax_i+b) = ((a+a')x_i+(b+b'))=((a+a')x_i+(b+b'))$
 \end{proof}

\paragraph{Proof of Lemma 4.3 (Continuous version)}
\begin{proof}
 Let $x_i$ be a data-point. As in the discrete case, we have two options. In the trivial case, the datapoint is inactive in the cell, in the sense that for all $(a,b,c)$ in the cell, $c(a x_i +b)_+=0$. In this case the equality is always true. In the non-zero (linear) case, all neurons in the cell are active. In this case we may omit the ReLu itself, and the equality we need to show is:

 \begin{equation}
 \int c (a x_i + b) d\mu = \int c (a x_i + b) d\mu'
 \end{equation}	

 by linearity of the integral this is equivalent to

 \begin{equation}
 x_i \cdot \int c a  d\mu + \int cb d\mu=  x_i \cdot \int c a  d\mu' + \int cb d\mu'
 \end{equation}	

 and now the result follows directly from the hypothesis.
\end{proof}

\paragraph{Proof of Lemma 4.4}

	\begin{lemma}
	Let $(a,b,c), (a’,b’,c’)$ belong to the same cell and not represent the same neuron, let $\mu = \frac 1 2 (\delta_{a,b,c}+\delta_{a’,b’,c’})$. Then there is a probability measure $\mu’$ concentrated at a single point $(\alpha, \beta,\gamma)$ such that estimates (\ref{eq:moment_estimates}) hold for $\mu,\mu’$ but $\langle V, \mu\rangle < \langle V, \mu’\rangle $. \footnote{There is a typo in the statement of lemma 4.4 in the main text in how the interpolation in the lifted space is performed; it has been addressed here.}
	\end{lemma}
\begin{proof}
Assume without loss of generality that $c,c'$ have the same sign (since otherwise we can transfer mass from one neuron to the other). 
Since $(a,b,c)$ and $(a',b',c')$ do not represent the same neuron, 
by the effective strict convexity of $w$ we have 
\begin{equation}
\label{eq:mim}
w( \lambda a + (1-\lambda) a', \lambda b + (1-\lambda)b') < \lambda w(a,b) + (1-\lambda)w(a',b')    
\end{equation}
for any $\lambda \in (0,1)$. 

Pick $\lambda = \frac{|c|}{|c|+|c'|} $.
Then we easily verify that $(\alpha, \beta) = \lambda (a,b) + (1-\lambda)(a',b')$ and $\gamma = c + c'$ verify the moment estimates, and 
$2|\gamma|w(\alpha,\beta) < |c| w(a,b) + |c'| w(a',b')$ by (\ref{eq:mim}). 


\end{proof}

\paragraph{Proof of Lemma 4.5}
\begin{proof}
	Assume minimizers exist, as otherwise the result is vacuously true. 
	
	Given a minmizer $\tilde \nu$ there is a measure $\nu_0 = \delta_{(a_\nu,b_{\nu},c_\nu)}$ that has the same mass and first moment as $\tilde\nu$, but is concentrated at a single point $(\alpha,\beta,\pm1)$ (the center of mass). By convexity, this $\nu_0$ has to be a minimizer of the loss amongst all functions with the same first moment estimate (by Jensen) so it suffices to show that $\nu_0$ is effectively unique. This is a consequence of effective convexity: If $\nu'$ is another minimizer, $\nu'$ has to be supported on parameters that represent the same parameter as $\nu_0$, because otherwise we can build an even better measure:
	
	Let 
	$$\phi(a,b,c):= 
	\left(\frac 1 2 (\alpha + |c|a),\frac 1 2 (\beta + |c|b), \pm 1\right)
	$$
	Then, for any given measure $\mu$ with the same moment estimates as $\nu_0$, $\phi_*\mu$ has the same moment estimates as $\mu$ and $\nu_0$, but unless $\mu$ is effectively supported at the same point as $\nu_0$, it must be (by strict convexity) that
	\begin{equation}
	    \langle w,\nu_0\rangle > \langle w, \phi_*\mu \rangle
	\end{equation}
	which contradicts the hypothesis that $\nu_0$ is a minimizer.
\end{proof}

\paragraph{Proof of Theorem 4.7}

\begin{proof}
    Let $\nu,\nu'$ be two minimizing measures. By convexity, $\frac 1 2 (\nu, \nu')$ must also be a minimizing sequence. Lemma \ref{lem:one_point_per_cell} guarantees then that $\frac 1 2 (\nu, \nu')$ will be supported in at most one point per cell. Therefore, if $\nu,\nu'$ had support on the same cell, it must be exactly on the same point.
\end{proof}

\paragraph{Proof of Corollary 4.8}
\begin{proof}
Let $\mu$ be a minimizer to the penalized problem

\begin{equation}
    \min_{\nu \in \mathcal P} \langle |c|w(a,b), \nu(a,b,c)\rangle + L(\{|f_\nu(x_i)-y_i|\}_{i=1}^n)
\end{equation}
for some similarity loss function $L$. Then $\mu$ must also be the minimizer to the following \text{restricted problem} (note that this is not the traditional penalized-constrained duality):

\begin{equation}
    \label{eq:fake_problem}
    \min_{\nu \in \mathcal P, f_\nu(x_i) = f_\mu(x_i)} \langle |c|w(a,b), \nu(a,b,c)\rangle + L(\{|f_\nu(x_i)-y_i|\}_{i=1}^n)
\end{equation}

this is because a minimizer must remain a minimizer if we further constrain the feasible set.

This is again the constrained problem in equation \eqref{eq:erm}, using $f_\mu(x_i)$ instead of $y_i$. In this case Theorem \ref{thm:weight_implies_sparsity} states that solutions to this problem must be sparse. 
\end{proof}

\end{document}